\documentclass[11pt]{article}

\usepackage[margin=1in]{geometry}
\usepackage{amsmath,amssymb,amsthm,mathtools}
\usepackage{enumitem}
\usepackage[hidelinks]{hyperref}

    \usepackage{graphicx, tikz}
        \usetikzlibrary{arrows.meta, positioning}
        
    \usepackage{mathtools, amsthm, amsfonts, amsmath, amssymb, nicefrac}
    \usepackage{enumitem}
        \setlist[itemize,enumerate]{
          nosep,
          topsep=0pt,
          partopsep=0pt,
          parsep=0pt,
          itemsep=0pt,
          leftmargin=1.5em,
          labelsep=0.4em
        }
    \usepackage{fullpage}
    \usepackage[numbers,sort&compress]{natbib}
    \usepackage{hyperref}
    \usepackage{cleveref} %TC: Keep this last!!!

    \newtheorem{corollary}{Corollary}
    \newtheorem{definition}{Definition}
    \newtheorem{proposition}{Proposition}
    
    \newtheorem{lemma}{Lemma}
    \newtheorem{remark}{Remark}
    \newtheorem{theorem}{Theorem}

    \newcommand{\cF}{\mathcal{F}}

    \newcommand{\E}{\mathbb{E}}

    \newcommand{\Pb}{\mathbb{P}}
    \newcommand{\lrb}[1]{\left(#1\right)}

    \newcommand{\lsb}[1]{\left[#1\right]}

    \DeclareSymbolFont{extraup}{U}{zavm}{m}{n}
    \DeclareMathSymbol{\clubsuit}{\mathalpha}{extraup}{84}
    \DeclareMathSymbol{\spadesuit}{\mathalpha}{extraup}{81}
    \DeclareMathSymbol{\varheartsuit}{\mathalpha}{extraup}{86}
    \DeclareMathSymbol{\vardiamondsuit}{\mathalpha}{extraup}{87}

\newcommand{\Reg}{\operatorname{Reg}}
\newcommand{\dmin}{\Delta_{\min}}

\title{Optimal Gap-Dependent Regret for Private Stochastic Decision-Theoretic Online Learning}
\author{
Tommaso Cesari\\
School of Electrical Engineering and Computer Science\\
University of Ottawa\\
Ottawa, Canada\\
\texttt{tcesari@uottawa.ca}
\and
Roberto Colomboni\\
School of Mathematics\\
University of Bristol\\
Bristol, United Kingdom\\
\texttt{roberto.colomboni@bristol.ac.uk}
}

\begin{document}

\maketitle

\begin{abstract}
We study stochastic decision-theoretic online learning with full information and event-level pure differential privacy.
A COLT open problem of \citeauthor{hu2024} \cite{hu2024} asks to determine the optimal gap-dependent regret rate for stochastic decision-theoretic online learning under pure event-level differential privacy.
For $K$ actions, losses in $[0,1]$, and a unique best action separated from the second-best action by gap $\dmin$, the known lower bound is of order
$
  \frac{\log K}{\min\{\dmin,\varepsilon\}},
$
or equivalently, up to universal constants, of order
\[
  \frac{\log K}{\dmin}+\frac{\log K}{\varepsilon}.
\]
We give a horizon-free pure-DP algorithm and prove the explicit regret bound
\[
  \Reg_T
  \le
  1000 \cdot \left(\frac{\log K}{\dmin}+\frac{\log K}{\varepsilon}\right)
\]
for every horizon $T$.
The numerical constant is not optimized.
The algorithm partitions time into blocks of exponentially increasing size, plays a single action throughout each block, and chooses the next action by an exponential mechanism applied to a data-independent random prefix of the previous block.
The random prefix converts block regret into a sum, over all prefix lengths, of softmax selection errors.
A single entropy-potential argument controls all privacy-dominated large-gap actions at cost $\log K/\varepsilon$.
\end{abstract}

\section{Introduction}

Decision-theoretic online learning \cite{freund1997} with stochastic full-information feedback is among the cleanest models where one can study the statistical cost of online decision making.
There are $K$ actions.
At each round, the learner chooses one action, suffers its loss, and observes the full loss vector.
When the loss vectors are i.i.d.\ and the best action is separated from the others by a gap, the optimal non-private gap-dependent regret is $\Theta(\log K/\dmin)$ \cite{kotlowski2018,mourtada2019}.

The differentially private version is more delicate.
\citeauthor{hu2021} \cite{hu2021} proved a lower bound of order
\[
  \frac{\log K}{\min\{\dmin,\varepsilon\}},
\]
which is equivalent up to universal constants to $\log K/\dmin+\log K/\varepsilon$.
\citeauthor{hu2024} \cite{hu2024} isolated the remaining question as a COLT open problem: determine the optimal instance-dependent regret rate under pure event-level differential privacy.
\citeauthor{wu2026} \cite{wu2026} recently made substantial progress and partially addressed the open problem by removing the horizon dependence and proving
\[
  O\left(\frac{\log K}{\dmin}+\frac{\log^2 K}{\varepsilon}\right)
\]
for the stochastic problem, together with a tight $\Theta(\log K/\varepsilon)$ analysis in a deterministic special case.

We close the remaining logarithmic gap in the stochastic problem.
Our algorithm is a randomized-prefix version of private Follow-The-Leader.
Time is partitioned into dyadic blocks.
During a block, the learner repeats one action.
At the end of the block, the learner samples a random prefix length from the second half of the block and applies the exponential mechanism to the cumulative losses over that prefix.
The selected action is then used throughout the next block.
Because each event is used only for the private selection of the following block, privacy follows from a direct factorization of the output law.
We give a direct factorization proof rather than relying on parallel composition.

The proof has two ingredients.
First, randomizing the prefix length removes the dyadic clock from the analysis: the regret over all blocks is bounded by a constant times
\[
  \sum_{m\ge 1}
  \E\left[\sum_{j=1}^K \Delta_j P_{m,j}\right],
\]
where $P_{m,j}$ is the random softmax probability of action $j$ after $m$ i.i.d.\ samples.
Second, this prefix-length softmax-error sum splits into small-gap and large-gap terms.
The small-gap part is controlled by Hoeffding's inequality and dyadic grouping, giving $\log K/\dmin$.
The large-gap part is controlled without summing over actions separately: a restricted softmax potential telescopes and costs only its initial entropy, $\log K$, divided by the privacy scale.
This is the step that removes the extra logarithm.

\emph{Thus, the proof also gives a conceptual explanation for the two terms in the optimal rate.}
The term $\log K/\dmin$ is the statistical cost of identifying the best action among near-optimal competitors.
The term $\log K/\varepsilon$ is the privacy cost of privately eliminating clearly suboptimal actions: once the gaps are large, concentration is no longer the bottleneck, and the only remaining resource being spent is the initial softmax entropy.
\emph{In this sense, the proof does not merely match the lower bound; it separates the statistical price from the privacy price and shows exactly where each one enters.}

\paragraph{Contributions.}
We give a constructive horizon-free $\varepsilon$-DP algorithm for stochastic DTOL and prove the explicit regret bound
\[
  \Reg_T
  \le
  1000\cdot\left(\frac{\log K}{\dmin}+\frac{\log K}{\varepsilon}\right).
\]
Together with the lower bound of \cite{hu2021}, this resolves the gap-dependent pure-DP rate posed by \cite{hu2024}.
Although the open problem is stated for the usual product-over-actions stochastic model, our proof only uses independence across time.
Thus it applies to i.i.d.\ loss vectors with arbitrary dependence among coordinates within each round; the product-over-actions model is a special case.

\section{Related Work}

\paragraph{Decision-theoretic online learning.}
Decision-theoretic online learning was introduced by \citeauthor{freund1997} \cite{freund1997} as a general full-information online decision model.
In the stochastic regime, Follow-The-Leader and exponentially weighted algorithms exhibit gap-dependent behavior governed by the separation between the best and second-best actions.
The sharp non-private rate is $\Theta(\log K/\dmin)$; see, for instance, the Follow-The-Leader upper bounds of \citeauthor{kotlowski2018} \cite{kotlowski2018} and the optimality results of \citeauthor{mourtada2019} \cite{mourtada2019}.
Our result preserves this statistical term and identifies the additional cost forced by pure event-level differential privacy.

\paragraph{Differential privacy and private selection.}
Differential privacy was introduced by \citeauthor{dwork2006} \cite{dwork2006}.
Our algorithm uses the exponential mechanism of \citeauthor{mcsherry2007} \cite{mcsherry2007} to select the next action after each block, by sampling from a softmax distribution over cumulative losses.
In contrast with standard one-shot private selection, the online setting requires the entire released action prefix to be private.
Our block construction is designed so that each event affects only one private selection.
The played actions inside the following block are deterministic post-processing of that selection.

\paragraph{Private stochastic online learning.}
\citeauthor{hu2021} \cite{hu2021} studied the differentially private stochastic online-learning setting most relevant here, including the full-information stochastic DTOL problem.
For full-information stochastic DTOL, their lower bound is of order
\[
  \frac{\log K}{\min\{\dmin,\varepsilon\}}.
\]
This is equivalent up to universal constants to $\log K/\dmin+\log K/\varepsilon$.
Their upper bound matches the statistical term but can pay an additional $\log T$ factor in the privacy term.
\citeauthor{hu2024} \cite{hu2024} then formulated the remaining gap as a COLT open problem: determine the optimal gap-dependent regret rate for stochastic DTOL under pure differential privacy.

\paragraph{Partial progress on the open problem.}
\citeauthor{wu2026} \cite{wu2026} partially addressed the open problem.
They removed the horizon dependence in the stochastic upper bound and obtained
\[
  O\left(\frac{\log K}{\dmin}+\frac{\log^2 K}{\varepsilon}\right),
\]
and they proved the tight $\Theta(\log K/\varepsilon)$ rate in a deterministic special case.
Their stochastic analysis still pays an additional $\log K$ in the privacy term.
At a technical level, their proof uses Bernoulli resampling and a monotonicity property, followed by a summation over actions that leaves the extra logarithm.

\paragraph{This work.}
The present algorithm remains in the pure-DP, stochastic, full-information DTOL setting of \cite{hu2024,wu2026}.
The new ingredient is a data-independent random prefix in each dyadic block.
This randomization converts block regret into a prefix-length sum of softmax selection errors.
For large-gap actions, the sum is controlled directly by a single softmax potential, whose initial value is at most $\log K$.
This removes the remaining logarithmic factor and yields the matching privacy term $\log K/\varepsilon$.

\section{Model}

Fix an integer $K\ge 2$ and write $[K]=\{1,\ldots,K\}$.
All random variables are defined on one probability space $(\Omega,\cF,\Pb)$.
The loss process is a sequence $(X_t)_{t\ge 1}$ of i.i.d.\ random variables with values in $[0,1]^K$.
We write
\[
  X_t=(X_{t,1},\ldots,X_{t,K}).
\]
The coordinates of $X_t$ may be dependent.
Only independence across different times is used.

At round $t$, the learner chooses an action $I_t\in[K]$ before observing $X_t$, suffers loss $X_{t,I_t}$, and then observes the whole vector $X_t$.
The action $I_t$ is therefore measurable with respect to the learner's internal randomization and the past loss vectors $X_1,\ldots,X_{t-1}$.
The full-information assumption is important here: the future observations do not depend on which action was played.

Let
\[
  \mu_j=\E[X_{1,j}],\qquad j\in[K].
\]
After relabeling actions, assume
\[
  \mu_1<\mu_2\le\cdots\le\mu_K.
\]
Define the pointwise expected-loss gaps
\[
  \Delta_1=0,
  \qquad
  \Delta_j=\mu_j-\mu_1\quad (j\in\{2,\ldots,K\}),
\]
and define
\[
  \dmin=\Delta_2.
\]
Since losses belong to $[0,1]$, one has $0<\dmin\le 1$ and $0\le \Delta_j\le 1$ for all $j$.

The expected pseudoregret up to horizon $T$ is
\[
  \Reg_T
  =
  \E\left[\sum_{t=1}^T \bigl(X_{t,I_t}-X_{t,1}\bigr)\right] =
  \sum_{t=1}^T \E[\Delta_{I_t}].
\]

\begin{definition}[Event-level pure differential privacy]
Fix $t\ge 1$.
Two deterministic prefixes $x_{1:t},x'_{1:t}\in([0,1]^K)^t$ are called neighboring if they differ in at most one whole loss vector.
A randomized online algorithm is $\varepsilon$-differentially private if, for every $t\ge 1$, every neighboring pair $x_{1:t},x'_{1:t}$, and every subset $S\subseteq[K]^t$,
\[
  \Pb\bigl((I_1,\ldots,I_t)\in S\mid x_{1:t}\bigr)
  \le
  e^\varepsilon
  \Pb\bigl((I_1,\ldots,I_t)\in S\mid x'_{1:t}\bigr).\footnote{The ``conditioning'' notation $\mid x_{1:t}$ means that the input prefix is fixed equal to $x_{1:t}$ and that probability is taken only over the algorithm's internal randomization.}
\]
Equivalently, for every fixed input prefix, the algorithm defines a probability measure on $[K]^t$, and the two measures associated with neighboring prefixes are within multiplicative factor $e^\varepsilon$ on every event.
\end{definition}

\section{Algorithm and Main Result}

All equalities involving conditional expectations are understood $\Pb$-almost surely.

Fix $\varepsilon>0$ and set
\[
  \eta=\min\{\varepsilon/2,1/8\}.
\]
For each $r\ge 0$, define the dyadic block
\[
  B_r=\{2^r,2^r+1,\ldots,2^{r+1}-1\}.
\]
Thus $|B_r|=2^r$.
The algorithm plays one action throughout each block.
Let $A_r$ be the action played on every round in $B_r$.
The initial action $A_0$ is sampled from any distribution on $[K]$ that does not depend on the loss sequence.
For concreteness, one may take $A_0$ uniform on $[K]$.

After block $B_r$ has been observed, the algorithm samples the action $A_{r+1}$ used in the next block.
For $r=0$, set $M_0=1$.
For $r\ge 1$, the learner samples an internal random index
\[
  M_r\sim \operatorname{Unif}\{2^{r-1}+1,\ldots,2^r\},
\]
independently of the loss vectors and of all other randomization.
For $s\in\{1,\ldots,2^r\}$ write
\[
  X^{(r,s)}=X_{2^r-1+s}.
\]
This is the $s$th loss vector in block $B_r$.
For $j\in[K]$ and $m\in\{1,
\ldots,2^r\}$ define the cumulative block-prefix loss
\[
  L_{r,j}(m)=\sum_{s=1}^m X^{(r,s)}_j.
\]
Given $M_r$ and the observed prefix $X^{(r,1)},\ldots,X^{(r,M_r)}$, the next action is sampled according to the exponential-weights distribution
\[
  \Pb(A_{r+1}=j\mid M_r,X^{(r,1)},\ldots,X^{(r,M_r)})
  =
  \frac{\exp(-\eta L_{r,j}(M_r))}
       {\sum_{i=1}^K \exp(-\eta L_{r,i}(M_r))}.
\]
Equivalently, this is the softmax of the negative cumulative losses
$-L_{r,j}(M_r)$ with inverse temperature $\eta$.

\begin{theorem}[Explicit gap-dependent private regret]\label{thm:main-explicit}
For every $K\ge 2$, every $\varepsilon>0$, every horizon $T\ge 1$, and every i.i.d.\ loss process with marginal distribution supported on $[0,1]^K$ with a unique best action, the randomized-prefix softmax algorithm is $\varepsilon$-differentially private and satisfies
\begin{equation}
\label{eq:regret_bound}    
  \Reg_T
  \le
  1
  +
  800\frac{\log K}{\dmin}
  +
  16\frac{\log K}{\eta},
  \qquad
  \eta=\min\{\varepsilon/2,1/8\}.
\end{equation}

In particular,
\begin{equation}
    \label{eq:main_bound}
  \Reg_T
  \le
  1000\left(
      \frac{\log K}{\dmin}
      +
      \frac{\log K}{\varepsilon}
    \right).
\end{equation}
The constant $1000$ is explicit and universal, but it is not optimized.
\end{theorem}

\begin{corollary}[Matching rate]
Combining \Cref{thm:main-explicit} with the lower bound of \cite{hu2021} gives the optimal gap-dependent rate
\[
  \Theta\left(
      \frac{\log K}{\dmin}
      +
      \frac{\log K}{\varepsilon}
    \right).
\]
\end{corollary}

\begin{remark}[Relation to \cite{wu2026}]
\cite{wu2026} prove the stochastic upper bound
\[
  O\left(\frac{\log K}{\dmin}+\frac{\log^2 K}{\varepsilon}\right).
\]
Their analysis uses Bernoulli resampling to obtain a monotonicity property and then sums an early-time contribution over action indices, yielding the remaining $\log K$ factor.
Our randomized-prefix clock instead exposes a softmax potential over prefix lengths.
For the privacy-dominated large-gap actions, the potential telescopes before any per-action harmonic summation is needed, leaving only one $\log K$.
\end{remark}

\section{Proofs}

\subsection{Auxiliary facts used in the proof}

\begin{lemma}[Hoeffding bound with the constants used below]\label{lem:hoeffding}
Let $U_1,\ldots,U_m$ be independent random variables with values in $[-1,1]$ and common expectation $\E[U_s]=\mu>0$.
Then
\[
  \Pb\left(\sum_{s=1}^m U_s\le \frac{m\mu}{2}\right)
  \le
  \exp\left(-\frac{m\mu^2}{8}\right).
\]
\end{lemma}

\begin{proof}
Hoeffding's inequality for independent variables $U_s\in[a_s,b_s]$ states that, for every $u>0$,
\[
  \Pb\left(\sum_{s=1}^m (U_s-\E[U_s])\le -u\right)
  \le
  \exp\left(-\frac{2u^2}{\sum_{s=1}^m (b_s-a_s)^2}\right).
\]
This is \cite[Theorem 2]{hoeffding1963probability}.
Here $a_s=-1$, $b_s=1$, and $u=m\mu/2$.
Therefore
\[
  \frac{2u^2}{\sum_{s=1}^m(b_s-a_s)^2}
  =
  \frac{2(m\mu/2)^2}{4m}
  =
  \frac{m\mu^2}{8},
\]
which proves the claim.
\end{proof}

\begin{lemma}[Elementary inequalities]\label{lem:elementary}
The following inequalities hold.
\begin{enumerate}[label=(\roman*)]
\item $1-e^{-x}\ge x/2$, for every $x\in[0,1]$. \label{eq:1}
\item $e^{-x}\le 1-x/2$, for every $x\in[0,1]$.\label{eq:2}
\item $\log(1-u)\le -u$, for every $u\in[0,1)$.\label{eq:3}
\item $
  \sum_{m>n}e^{-mx}
  \le
  \frac{2e^{-nx}}{x}
    $, for every $x\in(0,1]$ and every integer $n\ge 0$.\label{eq:4}
\item $\log(K+1)\le 2\log K$, for every integer $K\ge 2$.\label{eq:5}
\end{enumerate}
\end{lemma}

\begin{proof}
For (i), define $g(x)=1-e^{-x}-x/2$.
Then $g(0)=0$, $g'(x)=e^{-x}-1/2$, and $g(1)=1-e^{-1}-1/2>0$.
The function increases on $[0,\log 2]$ and decreases on $[\log 2,1]$, so its minimum on $[0,1]$ is nonnegative.
This proves (i).
Item (ii) is the same inequality rearranged.
For (iii), define $h(u)=-u-\log(1-u)$.
Then $h(0)=0$ and $h'(u)=u/(1-u)\ge 0$ for $u\in[0,1)$.
For (iv),
\[
  \sum_{m>n}e^{-mx}
  =
  \frac{e^{-(n+1)x}}{1-e^{-x}}
  \le
  \frac{e^{-nx}}{1-e^{-x}}
  \le
  \frac{2e^{-nx}}{x},
\]
where the last inequality uses (i).
For (v), $K+1\le K^2$ for $K\ge2$, hence $\log(K+1)\le2\log K$.
\end{proof}

\subsection{Privacy}

We first prove privacy for the one-block selector as an ordinary randomized map on deterministic datasets.

\begin{definition}[Differential privacy for a randomized map]
Fix an integer $n\ge1$.
A randomized map $\mathsf M:([0,1]^K)^n\to[K]$ is $\alpha$-differentially private if, for every pair of deterministic datasets $d,d'\in([0,1]^K)^n$ that differ in at most one coordinate vector and every $S\subseteq[K]$,
\[
  \Pb(\mathsf M(d)\in S)
  \le
  e^\alpha\Pb(\mathsf M(d')\in S).
\]
\end{definition}

\begin{lemma}[Privacy of the softmax map]
\label{lem:block-privacy}
Fix $n\ge 1$, and let $N$ be a random variable on $\{1,\ldots,n\}$ whose
distribution does not depend on the dataset.
Given a deterministic dataset $d=(x_1,\ldots,x_n)\in([0,1]^K)^n$, define
\[
  L_j^d(m)=\sum_{s=1}^m x_{s,j},
  \qquad
  j\in[K],
  \quad
  m\in\{1,\ldots,n\}.
\]
After sampling $N$, output $J\in[K]$ with conditional law
\[
  \Pb(J=j\mid N=m,d)
  =
  \frac{\exp(-\eta L_j^d(m))}
       {\sum_{i=1}^K \exp(-\eta L_i^d(m))}.
\]
Then the randomized map $d\mapsto J$ is $2\eta$-differentially private in the
sense of Definition~2.
\end{lemma}

\begin{proof}
Let $d=(x_1,\ldots,x_n)$ and $d'=(x'_1,\ldots,x'_n)$ differ in at most one vector.
Fix $m\in\{1,\ldots,n\}$.
For every $j\in[K]$,
\[
  |L_j^d(m)-L_j^{d'}(m)|\le 1,
\]
because the two prefixes either are identical or differ in one vector whose $j$th coordinate lies in $[0,1]$.
Write
\[
  Z^d(m)\coloneqq \sum_{i=1}^K\exp(-\eta L_i^d(m)).
\]
For every $i$,
\[
  \exp(-\eta L_i^{d'}(m))
  \le
  e^\eta \exp(-\eta L_i^d(m)),
\]
and therefore
\[
  Z^{d'}(m)\le e^\eta Z^d(m).
\]
For every $j\in[K]$,
\[
\begin{aligned}
  \frac{\Pb(J=j\mid N=m,d)}{\Pb(J=j\mid N=m,d')}
  &=
  \exp\{-\eta(L_j^d(m)-L_j^{d'}(m))\}
  \frac{Z^{d'}(m)}{Z^d(m)} 
  \le e^\eta e^\eta
  =e^{2\eta}.
\end{aligned}
\]
Summing this pointwise bound over $j\in S$ gives
\[
  \Pb(J\in S\mid N=m,d)
  \le
  e^{2\eta}\Pb(J\in S\mid N=m,d')
\]
for every $S\subseteq[K]$.
Now average over $N$.
Since the law of $N$ is the same under $d$ and $d'$,
\[
\begin{aligned}
  \Pb(J\in S\mid d)
  &=
  \sum_{m=1}^n \Pb(N=m)\Pb(J\in S\mid N=m,d) \\
  &\le
  e^{2\eta}
  \sum_{m=1}^n \Pb(N=m)\Pb(J\in S\mid N=m,d') \\
  &=
  e^{2\eta}\Pb(J\in S\mid d').
\end{aligned}
\]
This is exactly $2\eta$-differential privacy for the softmax map.
\end{proof}

\begin{proposition}[Global event-level privacy]\label{prop:global-privacy}
The randomized-prefix softmax algorithm is $\varepsilon$-differentially private.
\end{proposition}

\begin{proof}
Fix a horizon $t\ge1$, and let $s=s(t)$ be the unique integer such that
$t\in B_s$.
The action prefix $(I_1,\ldots,I_t)$ is a deterministic function of the block
actions
\[
  A_0,A_1,\ldots,A_s .
\]
Indeed, if $u\in B_r$ and $u\le t$, then $I_u=A_r$.

Fix a deterministic prefix $y_{1:t}$.
Since $A_0$ is sampled independently of the data, and since for each
$\ell\in\{0,\ldots,s-1\}$ the action $A_{\ell+1}$ is sampled using only the data
in block $B_\ell$, the law of $(A_0,\ldots,A_s)$ under the fixed input prefix
$y_{1:t}$ factorizes as follows: for every $a_{0:s}\in[K]^{s+1}$,
\[
\begin{aligned}
&\Pb\lrb{(A_0,\ldots,A_s)=a_{0:s}\mid y_{1:t}}
\\
&\quad =
\Pb\lrb{A_0=a_0}
\prod_{\ell=0}^{s-1}
\Pb\lrb{
  A_{\ell+1}=a_{\ell+1}\mid y_{B_\ell}
}.
\end{aligned}
\]
Here $y_{B_\ell}$ denotes the restriction of the deterministic prefix $y_{1:t}$
to the block $B_\ell$.
The product is empty when $s=0$.

Now fix two deterministic prefixes $x_{1:t}$ and $x'_{1:t}$ that differ in at
most one time index.
If the differing vector lies in the current block $B_s$, then it does not appear
in any of the blocks $B_0,\ldots,B_{s-1}$.
Hence all factors in the above product are identical for $x_{1:t}$ and
$x'_{1:t}$, and therefore the two laws of $(A_0,\ldots,A_s)$ are identical.

Otherwise, the differing vector lies in a previous block $B_r$, for some
$r\le s-1$.
Then all factors in the product are identical except possibly the factor
\[
  \Pb\lrb{A_{r+1}=a_{r+1}\mid x_{B_r}}.
\]
By \Cref{lem:block-privacy}, applied to the two deterministic block datasets
$x_{B_r}$ and $x'_{B_r}$, changing one loss vector in $B_r$ changes the law of
$A_{r+1}$ by at most a multiplicative factor $e^{2\eta}$.
In particular, for every $a_{r+1}\in[K]$,
\[
  \Pb\lrb{A_{r+1}=a_{r+1}\mid x_{B_r}}
  \le
  e^{2\eta}
  \Pb\lrb{A_{r+1}=a_{r+1}\mid x'_{B_r}}.
\]
Combining this inequality with the product factorization gives, for every
$a_{0:s}\in[K]^{s+1}$,
\[
  \Pb\lrb{(A_0,\ldots,A_s)=a_{0:s}\mid x_{1:t}}
  \le
  e^{2\eta}
  \Pb\lrb{(A_0,\ldots,A_s)=a_{0:s}\mid x'_{1:t}}.
\]
Summing over $a_{0:s}\in E$ yields, for every $E\subseteq[K]^{s+1}$,
\[
  \Pb\lrb{(A_0,\ldots,A_s)\in E\mid x_{1:t}}
  \le
  e^{2\eta}
  \Pb\lrb{(A_0,\ldots,A_s)\in E\mid x'_{1:t}}.
\]

Finally, define the deterministic map $\phi_t:[K]^{s+1}\to[K]^t$ by
\[
  \phi_t(a_0,\ldots,a_s)(u)=a_{s(u)},
  \qquad u\in[t],
\]
where $s(u)$ is the unique integer such that $u\in B_{s(u)}$.
Then
\[
  (I_1,\ldots,I_t)=\phi_t(A_0,\ldots,A_s).
\]
For any $S\subseteq[K]^t$, apply the preceding inequality with
$E=\phi_t^{-1}(S)$.
This gives
\[
  \Pb\lrb{(I_1,\ldots,I_t)\in S\mid x_{1:t}}
  \le
  e^{2\eta}
  \Pb\lrb{(I_1,\ldots,I_t)\in S\mid x'_{1:t}}.
\]
Since $2\eta\le\varepsilon$, the algorithm is
$\varepsilon$-differentially private.
\end{proof}

\subsection{Randomized-prefix clock}

Let $(X_s)_{s\ge1}$ be the loss-vector process.
For $m\ge1$ and $j\in[K]$, define the cumulative loss
\[
  L_{m,j}=\sum_{s=1}^m X_{s,j}.
\]
For $j\in\{2,\ldots,K\}$ define the cumulative empirical gap
\[
  D_{m,j}=L_{m,j}-L_{m,1}
  =
  \sum_{s=1}^m (X_{s,j}-X_{s,1}).
\]
For $m\ge1$, define the random softmax weights
\[
  P_{m,1}
  \coloneqq
  \frac{1}{1+\sum_{i=2}^K e^{-\eta D_{m,i}}},
  \qquad
  P_{m,j}
  \coloneqq
  \frac{e^{-\eta D_{m,j}}}{1+\sum_{i=2}^K e^{-\eta D_{m,i}}}
  \quad (j\ge2).
\]
Finally define the deterministic number
\[
  F_m
  \coloneqq
  \E\left[\sum_{j=1}^K\Delta_jP_{m,j}\right]
  =
  \E\left[\sum_{j=2}^K\Delta_jP_{m,j}\right].
\]

\begin{lemma}[Clock reduction with explicit constant]\label{lem:clock}
For every horizon $T\ge1$,
\[
  \Reg_T
  \le
  1+4\sum_{m\ge1}F_m.
\]
\end{lemma}

\begin{proof}
The initial block $B_0=\{1\}$ has length one.
Since every gap is at most $1$, its regret contribution is at most $1$.

Fix $r\ge0$.
The action $A_{r+1}$ is played during block $B_{r+1}$.
If the horizon $T$ contains the whole block $B_{r+1}$, then this action is played
$|B_{r+1}|=2^{r+1}$ times.
If $T$ cuts the block $B_{r+1}$, then it is played fewer times.
In both cases, the number of rounds in which $A_{r+1}$ is played is at most
$2^{r+1}$.

Moreover, conditionally on $M_r=m$, the data used to sample $A_{r+1}$ are the
first $m$ loss vectors of block $B_r$.
Since the loss process is i.i.d.\ and $M_r$ is independent of the losses, these
$m$ vectors have the same distribution as $X_1,\ldots,X_m$. Therefore
\[
  \E[\Delta_{A_{r+1}}\mid M_r=m]=F_m.
\]
Taking expectation over $M_r$ gives
\[
  \E[\Delta_{A_{r+1}}]=\E[F_{M_r}].
\]
Thus the regret charged to block $B_{r+1}$ is at most
\[
  2^{r+1}\E[F_{M_r}].
\]
For $r=0$, $M_0=1$, hence this is $2F_1$, which is no greater than $4F_1$.
For $r\ge1$, $M_r$ is uniform on $\{2^{r-1}+1,\ldots,2^r\}$, a set of cardinality $2^{r-1}$.
Therefore
\[
  2^{r+1}\E[F_{M_r}]
  =
  2^{r+1}\frac{1}{2^{r-1}}
  \sum_{m=2^{r-1}+1}^{2^r}F_m
  =
  4\sum_{m=2^{r-1}+1}^{2^r}F_m.
\]
The intervals $\{2^{r-1}+1,\ldots,2^r\}$ for $r\ge1$ are disjoint and cover all integers $m\ge2$.
Adding the initial block contribution and all following block contributions gives
\[
  \Reg_T
  \le
  1+4F_1+4\sum_{m\ge2}F_m
  =
  1+4\sum_{m\ge1}F_m.
\]
\end{proof}

\subsection{Prefix-length softmax-error bound}

\begin{lemma}[Cumulative softmax-error bound with explicit constants]\label{lem:master}
For $\eta\le1/8$,
\[
  \sum_{m\ge1}F_m
  \le
  200\frac{\log K}{\dmin}
  +
  4\frac{\log K}{\eta}.
\]
\end{lemma}

\begin{proof}
Recall that, for $m\ge1$,
\[
  F_m
  =
  \E\left[
    \sum_{j=2}^K \Delta_j P_{m,j}
  \right].
\]
Therefore
\begin{equation}
\label{eq:decomposition}
    \begin{aligned}
  \sum_{m\ge1}F_m
  &=
  \sum_{m\ge1}
  \E\left[
    \sum_{\substack{j \in \{2,\dots,K\} \\\Delta_j\le 4\eta}}
    \Delta_jP_{m,j}
  \right]
  +
  \sum_{m\ge1}
  \E\left[
    \sum_{\substack{j \in \{2,\dots,K\} \\\Delta_j>4\eta}}
    \Delta_jP_{m,j}
  \right].
\end{aligned}
\end{equation}

We bound the two terms separately.

\paragraph{Small-gap contribution.}
If there is no gap $\Delta_j \le 4\eta$, then the first sum is trivially bounded by $0$.
Otherwise, fix $j\in\{2,\ldots,K\}$ with $\Delta_j\le4\eta$.
For every $m\ge1$,
\[
  P_{m,j}
  \le
  e^{-\eta D_{m,j}},
\]
because the denominator in the definition of $P_{m,j}$ is at least $1$.
Therefore
\[
\begin{aligned}
\E\lsb{P_{m,j}}
&=
\E\lsb{
  P_{m,j}
  \mathbb I\left\{
    D_{m,j}\le \frac{m\Delta_j}{2}
  \right\}
}
+
\E\lsb{
  P_{m,j}
  \mathbb I\left\{
    D_{m,j}> \frac{m\Delta_j}{2}
  \right\}
}
\\
&\le
\E\lsb{
  \mathbb I\left\{
    D_{m,j}\le \frac{m\Delta_j}{2}
  \right\}
}
+
\E\lsb{
  e^{-\eta D_{m,j}}
  \mathbb I\left\{
    D_{m,j}> \frac{m\Delta_j}{2}
  \right\}
}
\\
&=
\Pb\left(
  D_{m,j}\le \frac{m\Delta_j}{2}
\right)
+
\E\lsb{
  e^{-\eta D_{m,j}}
  \mathbb I\left\{
    D_{m,j}> \frac{m\Delta_j}{2}
  \right\}
}
\\
&\le
\Pb\left(
  D_{m,j}\le \frac{m\Delta_j}{2}
\right)
+
\E\lsb{
  e^{-\eta m\Delta_j/2}
  \mathbb I\left\{
    D_{m,j}> \frac{m\Delta_j}{2}
  \right\}
}
\\
&\le
\Pb\left(
  D_{m,j}\le \frac{m\Delta_j}{2}
\right)
+
e^{-\eta m\Delta_j/2}.
\end{aligned}
\]
The random variables $X_{s,j}-X_{s,1}$ are independent across $s$, take values in $[-1,1]$, and have expectation $\Delta_j$.
By \Cref{lem:hoeffding},
\[
  \Pb\left(D_{m,j}\le\frac{m\Delta_j}{2}\right)
  \le
  \exp\left(-\frac{m\Delta_j^2}{8}\right).
\]
Since $\Delta_j\le4\eta$, we have $\eta\ge\Delta_j/4$, and thus
\[
  \exp\left(-\frac{\eta m\Delta_j}{2}\right)
  \le
  \exp\left(-\frac{m\Delta_j^2}{8}\right).
\]
Consequently,
\begin{equation}\label{eq:small-action-tail}
  \E[P_{m,j}]
  \le
  2\exp\left(-\frac{m\Delta_j^2}{8}\right).
\end{equation}

For each integer $\ell\ge1$, define
\[
  H_\ell
  =
  \left\{j\in\{2,\ldots,K\}:
  \Delta_j\le4\eta,
  \quad
  2^{\ell-1}\dmin\le\Delta_j<2^\ell\dmin
  \right\}.
\]
Fix $\ell$ with $H_\ell\ne\varnothing$ and define
\[
  \tau_\ell
  =
  \left\lceil
  \frac{16\log(|H_\ell|+1)}{(2^{\ell-1}\dmin)^2}
  \right\rceil.
\]
For $m\le\tau_\ell$, using $\Delta_j<2^\ell\dmin$ for $j\in H_\ell$ and $\sum_{j\in H_\ell}P_{m,j}\le1$, we get the pathwise bound
\[
  \sum_{j\in H_\ell}\Delta_jP_{m,j}
  \le
  2^{\ell}\dmin
  \sum_{j\in H_\ell}P_{m,j}
  \le
  2^{\ell}\dmin.
\]
Therefore the early part of group $H_\ell$ is at most
\[
\begin{aligned}
  \sum_{m=1}^{\tau_\ell}
  \E\left[\sum_{j\in H_\ell}\Delta_jP_{m,j}\right]
  &\le
  2^{\ell}\dmin\tau_\ell  \\
  &\le
  \frac{32\log(|H_\ell|+1)}{2^{\ell-1}\dmin}
  +
  2^{\ell}\dmin \\
  &\le
  \frac{32\log(|H_\ell|+1)+2}{2^{\ell-1}\dmin}.
\end{aligned}
\]
The last step uses $2^{\ell-1}\dmin\le1$, which follows from $H_\ell\ne\varnothing$ and $\Delta_j\le1$.

For $m>\tau_\ell$, use \eqref{eq:small-action-tail} and $\Delta_j<2^{\ell}\dmin$ to obtain
\[
\begin{aligned}
  \sum_{m>\tau_\ell}
  \E\left[\sum_{j\in H_\ell}\Delta_jP_{m,j}\right]
  &\le
  \sum_{m>\tau_\ell}
  \sum_{j\in H_\ell}
  4(2^{\ell-1}\dmin)
  \exp\left(-\frac{m(2^{\ell-1}\dmin)^2}{8}\right).
\end{aligned}
\]
For nonempty $H_\ell$, the inequality $2^{\ell-1}\dmin\le\Delta_j\le4\eta$ and $\eta\le1/8$ imply $2^{\ell-1}\dmin\le1/2$.
Set $x=(2^{\ell-1}\dmin)^2/8$.
Then $x\in(0,1]$.
By \ref{eq:4} in \Cref{lem:elementary},
\[
  \sum_{m>\tau_\ell}
  \exp\left(-\frac{m(2^{\ell-1}\dmin)^2}{8}\right)
  \le
  \frac{16}{(2^{\ell-1}\dmin)^2}
  \exp\left(-\frac{\tau_\ell(2^{\ell-1}\dmin)^2}{8}\right).
\]
The definition of $\tau_\ell$ gives
\[
  \exp\left(-\frac{\tau_\ell(2^{\ell-1}\dmin)^2}{8}\right)
  \le
  \frac{1}{(|H_\ell|+1)^2}.
\]
Hence
\[
\begin{aligned}
  \sum_{m>\tau_\ell}
  \E\left[\sum_{j\in H_\ell}\Delta_jP_{m,j}\right]
  &\le
  \frac{64|H_\ell|}{2^{\ell-1}\dmin}
  \frac{1}{(|H_\ell|+1)^2} \le
  \frac{16}{2^{\ell-1}\dmin},
\end{aligned}
\]
because $a/(a+1)^2\le1/4$ for every integer $a\ge1$.
Combining the early and late parts yields
\[
  \sum_{m\ge1}
  \E\left[\sum_{j\in H_\ell}\Delta_jP_{m,j}\right]
  \le
  \frac{32\log(|H_\ell|+1)+18}{2^{\ell-1}\dmin}.
\]
Now sum over $\ell\ge1$.
Since $|H_\ell|\le K$ and $K\ge2$,
\begin{align}
\label{eq:small-total}
  \sum_{m\ge1}
  \E\left[
    \sum_{\substack{j \in \{2,\dots,K\} \\\Delta_j\le4\eta}}
    \Delta_jP_{m,j}
  \right]
  &=
\sum_{\substack{\ell \ge 1\\ H_\ell \neq \varnothing}}
  \sum_{m\ge1}\E\left[\sum_{j\in H_\ell}\Delta_jP_{m,j}\right]
  \nonumber\\
  &\le
  \sum_{\substack{\ell \ge 1\\ H_\ell \neq \varnothing}}
  \frac{32\log(|H_\ell|+1)+18}{2^{\ell-1}\dmin}
  \nonumber\\
  &\le
  \frac{1}{\dmin}
  \sum_{\ell\ge1}
  \frac{32\log(K+1)+18}{2^{\ell-1}} \nonumber \\
  &\le
  \frac{64\log(K+1)+36}{\dmin} \nonumber \\
  &\le
  200\frac{\log K}{\dmin},
\end{align}
where the last inequality uses \ref{eq:5} in \Cref{lem:elementary} and $36\le(36/\log2)\log K\le72\log K$ for $K\ge2$.
\paragraph{Large-gap contribution.}
If there is no gap $\Delta_j > 4\eta$, then the second sum is trivially bounded by $0$.
Assume then that there is at least one gap $\Delta_j$ such that $\Delta_j > 4\eta$.
For $m\ge0$ and $j\ge2$, define
\[
  D_{0,j}=0.
\]
For the large-gap actions $j$ with $\Delta_j>4\eta$, define
\[
  Z_m
  =
  1+\sum_{\substack{j \in \{2,\dots,K\} \\\Delta_j>4\eta}}e^{-\eta D_{m,j}}.
\]
Also define the random weights
\[
  Q_{m,1}=\frac{1}{Z_m},
  \qquad
  Q_{m,j}=\frac{e^{-\eta D_{m,j}}}{Z_m}
  \quad\text{for }j\in\{2,\ldots,K\}\text{ with }\Delta_j>4\eta.
\]
Finally define the random large-gap weighted gap
\[
  V_m
  =
  \sum_{\substack{j \in \{2,\dots,K\} \\\Delta_j>4\eta}}\Delta_j Q_{m,j}.
\]
Notice that
\[
  Q_{m,1}
  +
  \sum_{\substack{j \in \{2,\dots,K\} \\\Delta_j>4\eta}}
  Q_{m,j}
  =
  1,
  \qquad
  \sum_{\substack{j \in \{2,\dots,K\} \\\Delta_j>4\eta}}
  Q_{m,j}
  \le 1.
\]
For every large-gap action $j$, the denominator defining $P_{m,j}$ contains all terms appearing in $Z_m$ and possibly more positive terms.
Thus
\[
  P_{m,j}\le Q_{m,j}
  \qquad\text{for every }m\ge1\text{ and every }j\text{ with }\Delta_j>4\eta,
\]
which implies
\[
        \sum_{m\ge 1}\sum_{\substack{j \in \{2,\dots,K\} \\\Delta_j>4\eta}}\Delta_j P_{m,j}
    \le
        \sum_{m\ge 1}\sum_{\substack{j \in \{2,\dots,K\} \\\Delta_j>4\eta}}\Delta_j Q_{m,j}
    \le
        \sum_{m\ge 0}\sum_{\substack{j \in \{2,\dots,K\} \\\Delta_j>4\eta}}\Delta_j Q_{m,j}
    =
        \sum_{m\ge 0}V_m,
\]
and hence, taking expectations, noticing that all the addends are positive, and using Tonelli's theorem twice,
we get
\[
    \sum_{m\ge1}
  \E\left[
    \sum_{\substack{j \in \{2,\dots,K\} \\\Delta_j>4\eta}}
    \Delta_jP_{m,j}
  \right]
  \le
  \sum_{m\ge0}\E[V_m].
\]
We now prove that
\[
  \sum_{m\ge0}\E[V_m]
  \le
  4\frac{\log K}{\eta},
\]
to obtain
\begin{equation}
\label{eq:large-total}
    \sum_{m\ge1}
  \E\left[
    \sum_{\substack{j \in \{2,\dots,K\} \\\Delta_j>4\eta}}
    \Delta_jP_{m,j}
  \right]
  \le
  4\frac{\log K}{\eta}.
\end{equation}

For $m\ge0$ and every large-gap action $j$, set
\[
  Y_{m+1,j}=X_{m+1,j}-X_{m+1,1}.
\]
Then
\[
  D_{m+1,j}=D_{m,j}+Y_{m+1,j},
  \qquad
  Y_{m+1,j}\in[-1,1],
  \qquad
  \E[Y_{m+1,j}]=\Delta_j.
\]
For $m=0$, $\cF_0$ is the trivial sigma-algebra.
Let $\cF_m=\sigma(X_1,\ldots,X_m)$.
The weights $Q_{m,j}$ and $Z_m$ are $\cF_m$-measurable.
Moreover, since the loss vectors are i.i.d., $X_{m+1}$ is independent of $\cF_m$, and therefore
$Y_{m+1,j}$ is independent of $\cF_m$ for every $j\in[K]$.

Using
\[
  \frac{Z_{m+1}}{Z_m}
  =
  Q_{m,1}
  +
  \sum_{\substack{j \in \{2,\dots,K\} \\\Delta_j>4\eta}}
  Q_{m,j}\exp(-\eta Y_{m+1,j}),
\]
and the concavity of the logarithm, conditional Jensen gives
\[
\begin{aligned}
  \E\lsb{\log Z_{m+1}-\log Z_m\mid \cF_m}
  &=
  \E\lsb{
    \log\lrb{
      Q_{m,1}
      +
      \sum_{\substack{j \in \{2,\dots,K\} \\\Delta_j>4\eta}}
      Q_{m,j}\exp(-\eta Y_{m+1,j})
    }
    \mid \cF_m
  }
  \\
  &\le
  \log
  \E\lsb{
      Q_{m,1}
      +
      \sum_{\substack{j \in \{2,\dots,K\} \\\Delta_j>4\eta}}
      Q_{m,j}\exp(-\eta Y_{m+1,j})
    \mid \cF_m
  }
  \\
  &=
  \log \lrb{
      Q_{m,1}
      +
      \sum_{\substack{j \in \{2,\dots,K\} \\\Delta_j>4\eta}}
      Q_{m,j}\E\lsb{\exp(-\eta Y_{m+1,j})
    \mid \cF_m}
  }\\
  &=
  \log \lrb{
      Q_{m,1}
      +
      \sum_{\substack{j \in \{2,\dots,K\} \\\Delta_j>4\eta}}
      Q_{m,j}\E\lsb{\exp(-\eta Y_{m+1,j})
    }
  }.
\end{aligned}
\]

For each large-gap action $j$, Hoeffding's lemma \cite[Lemma 2.2]{boucheron2013concentration} for a random variable in an interval of length $2$ gives
\[
  \E[e^{-\eta Y_{m+1,j}}]
  \le
  \exp\left(-\eta\Delta_j+\frac{\eta^2}{2}\right).
\]
Since $\Delta_j>4\eta$, we have $\eta^2/2\le\eta\Delta_j/8$.
Thus
\[
  \E[e^{-\eta Y_{m+1,j}}]
  \le
  \exp\left(-\frac{7\eta\Delta_j}{8}\right).
\]
Because $\Delta_j\le1$ and $\eta\le1/8$, the number $\eta\Delta_j$ belongs to $[0,1]$.
Using \ref{eq:2} in \Cref{lem:elementary} with $x=7\eta\Delta_j/8$ gives
\[
  \exp\left(-\frac{7\eta\Delta_j}{8}\right)
  \le
  1-\frac{7\eta\Delta_j}{16}
  \le
  1-\frac{\eta\Delta_j}{4}.
\]
Consequently,
\[
\begin{aligned}
  Q_{m,1}
  +
  \sum_{\substack{j \in \{2,\dots,K\} \\\Delta_j>4\eta}}
  Q_{m,j}\E[e^{-\eta Y_{m+1,j}}]
  &\le
  Q_{m,1}
  +
  \sum_{\substack{j \in \{2,\dots,K\} \\\Delta_j>4\eta}}
  Q_{m,j}\left(1-\frac{\eta\Delta_j}{4}\right)
  =
  1-\frac{\eta V_m}{4}.
\end{aligned}
\]
Since $Q_{m,j}\ge0$, since
\[
Q_{m,1}+\sum_{j:\Delta_j>4\eta}Q_{m,j}=1,
\]
and since $0\le\Delta_j\le1$, we have
\[
0\le V_m
=
\sum_{j:\Delta_j>4\eta}\Delta_jQ_{m,j}
\le
\sum_{j:\Delta_j>4\eta}Q_{m,j}
\le1.
\]
Hence, because $\eta\le1/8$,
\[
0\le \frac{\eta V_m}{4}\le \frac{\eta}{4}\le\frac1{32}<1.
\]
By \ref{eq:3} of Lemma~2, applied with
$u=\eta V_m/4$, we get
\[
  \log\left(1-\frac{\eta V_m}{4}\right)
  \le
  -\frac{\eta V_m}{4}.
\]
Combining the previous displays gives
\[
  \E[\log Z_{m+1}-\log Z_m\mid\cF_m]
  \le
  -\frac{\eta V_m}{4}.
\]
Taking expectations,
\[
  \E[\log Z_{m+1}]-\E[\log Z_m]
  \le
  -\frac{\eta}{4}\E[V_m].
\]
Summing from $m=0$ to $N$ yields
\[
  \frac{\eta}{4}\sum_{m=0}^N\E[V_m]
  \le
  \E[\log Z_0]-\E[\log Z_{N+1}].
\]
Since $Z_{N+1}\ge1$, we have $\log Z_{N+1}\ge0$.
Since $D_{0,j}=0$,
\[
  Z_0
  =
  1+\left|\{j\in\{2,\ldots,K\}:\Delta_j>4\eta\}\right|
  \le
  K.
\]
Thus
\[
  \frac{\eta}{4}\sum_{m=0}^N\E[V_m]
  \le
  \log K.
\]
Letting $N\to\infty$ gives \eqref{eq:large-total}.

Combining \eqref{eq:small-total} and \eqref{eq:large-total} proves the lemma.
\end{proof}

\subsection{Proof of the main theorem}

\begin{proof}[Proof of \Cref{thm:main-explicit}]
Privacy follows from \Cref{prop:global-privacy}.
For regret, combine \Cref{lem:clock} and \Cref{lem:master}:
\[
\begin{aligned}
  \Reg_T
  \le
  1+4\sum_{m\ge1}F_m \le
  1
  +
  800\frac{\log K}{\dmin}
  +
  16\frac{\log K}{\eta}.
\end{aligned}
\]
This proves \eqref{eq:regret_bound}.
It remains only to express the bound in terms of $\varepsilon$.

If $\varepsilon\le1/4$, then $\eta=\varepsilon/2$, and hence
\[
  16\frac{\log K}{\eta}=32\frac{\log K}{\varepsilon}.
\]
Since $K\ge2$ and $\dmin\le1$,
\[
  1
  \le
  \frac{1}{\log2}\frac{\log K}{\dmin}
  \le
  2\frac{\log K}{\dmin}.
\]
Therefore
\[
  \Reg_T
  \le
  802\frac{\log K}{\dmin}
  +
  32\frac{\log K}{\varepsilon}
  \le
  1000\left(\frac{\log K}{\dmin}+\frac{\log K}{\varepsilon}\right).
\]

If $\varepsilon>1/4$, then $\eta=1/8$, and hence
\[
  16\frac{\log K}{\eta}=128\log K\le128\frac{\log K}{\dmin}.
\]
Again $1\le2\log K/\dmin$.
Thus
\[
  \Reg_T
  \le
  930\frac{\log K}{\dmin}
  \le
  1000\left(\frac{\log K}{\dmin}+\frac{\log K}{\varepsilon}\right).
\]
This proves \eqref{eq:main_bound}, and concludes the proof.
\end{proof}

\section{Why the Extra Logarithm Disappears}

\citet{wu2026} build on the dyadic report-noisy-max/follow-the-noisy-leader
algorithm of \citet{hu2021}.
The algorithm of \citet{hu2021} accumulates losses over an epoch and then uses
report-noisy-max with Laplace noise to choose the action played in the next epoch.
The analysis of \citet{wu2026} removes the horizon dependence from the regret bound,
but their proof still controls suboptimal actions separately through pairwise tail
bounds.
In the large-gap privacy-dominated regime, their per-action accounting produces a
harmonic summation over action indices, which yields the extra factor $\log K$ in
their privacy term.

Our randomized-prefix proof removes the per-action accounting.
After reducing the analysis from played rounds to block actions, the regret contribution of the privacy-dominated actions becomes a prefix-length sum of softmax errors:
\[
  \sum_{m\ge 1}
  \E\lsb{
    \sum_{j:\Delta_j>4\eta}
      \Delta_j P_{m,j}
  }.
\]
We control this prefix-length softmax sum through the restricted softmax potential
$\log Z_m$, rather than by separate tail bounds for each action.
The potential argument telescopes over $m$, and the whole large-gap contribution is
paid by the initial value
\[
  \log Z_0\le \log K,
\]
divided by the privacy scale $\eta$.

Consequently, our privacy term is of order $\log K/\eta$, and hence
$\log K/\varepsilon$, rather than $\log^2 K/\varepsilon$.
The random prefix is the mechanism behind the improvement: it turns the dyadic block regret into a prefix-length sum of softmax errors, and the potential argument controls this sum without action-by-action tail bounds.

\section{Conclusion}

The randomized-prefix clock turns the analysis of the dyadic private online-learning algorithm into a prefix-length sum of softmax errors.
This small change exposes an entropy potential that controls all large-gap actions simultaneously, giving the optimal $\log K/\varepsilon$ privacy cost instead of the $\log^2 K/\varepsilon$ cost left by previous stochastic analyses.
Together with the known lower bound, the result gives the optimal gap-dependent rate for pure-DP stochastic decision-theoretic online learning, up to the trivial cap by the horizon.

\bibliographystyle{plainnat}
\bibliography{biblio}

@book{boucheron2013concentration,
    author = {Boucheron, Stéphane and Lugosi, Gábor and Massart, Pascal},
    title = {Concentration Inequalities: A Nonasymptotic Theory of Independence},
    publisher = {Oxford University Press},
    year = {2013},
    month = {02},
    isbn = {9780199535255},
    doi = {10.1093/acprof:oso/9780199535255.001.0001},
    url = {https://doi.org/10.1093/acprof:oso/9780199535255.001.0001},
}

@article{hoeffding1963probability,
  title={Probability inequalities for sums of bounded random variables},
  author={Hoeffding, Wassily},
  journal={Journal of the American statistical association},
  volume={58},
  number={301},
  pages={13--30},
  year={1963},
  publisher={Taylor \& Francis}
}

@misc{dwork2006,
  title={Differential privacy. Automata, languages and programming},
  author={Dwork, Cynthia},
  year={2006},
  publisher={Springer, Berlin}
}

@inproceedings{mcsherry2007,
  title={Mechanism design via differential privacy},
  author={McSherry, Frank and Talwar, Kunal},
  booktitle={48th Annual IEEE Symposium on Foundations of Computer Science (FOCS'07)},
  pages={94--103},
  year={2007},
  organization={IEEE}
}

@article{freund1997,
  title={A decision-theoretic generalization of on-line learning and an application to boosting},
  author={Freund, Yoav and Schapire, Robert E},
  journal={Journal of computer and system sciences},
  volume={55},
  number={1},
  pages={119--139},
  year={1997},
  publisher={Elsevier}
}

@article{kotlowski2018,
  title={On minimaxity of follow the leader strategy in the stochastic setting},
  author={Kot{\l}owski, Wojciech},
  journal={Theoretical Computer Science},
  volume={742},
  pages={50--65},
  year={2018},
  publisher={Elsevier}
}

@article{mourtada2019,
  title={On the optimality of the hedge algorithm in the stochastic regime},
  author={Mourtada, Jaouad and Ga{\"\i}ffas, St{\'e}phane},
  journal={Journal of Machine Learning Research},
  volume={20},
  number={83},
  pages={1--28},
  year={2019}
}

@article{hu2021,
  title={Near-Optimal Algorithms for Differentially Private Online Learning in a Stochastic Environment},
  author={Hu, Bingshan and Huang, Zhiming and Mehta, Nishant A. and Hegde, Nidhi},
  journal={arXiv preprint arXiv:2102.07929},
  year={2021}
}

@InProceedings{hu2024,
  title = 	 {Open Problem: Optimal Rates for Stochastic Decision-Theoretic Online Learning Under Differentially Privacy},
  author =       {Hu, Bingshan and Mehta, Nishant A.},
  booktitle = 	 {Proceedings of Thirty Seventh Conference on Learning Theory},
  pages = 	 {5330--5334},
  year = 	 {2024},
  editor = 	 {Agrawal, Shipra and Roth, Aaron},
  volume = 	 {247},
  series = 	 {Proceedings of Machine Learning Research},
  month = 	 {30 Jun--03 Jul},
  publisher =    {PMLR},
  url = 	 {https://proceedings.mlr.press/v247/hu24a.html}
}

@InProceedings{wu2026,
  title = 	 {Improved Regret in Stochastic Decision-Theoretic Online Learning under Differential Privacy},
  author =       {Wu, Ruihan and Wang, Yu-Xiang},
  booktitle = 	 {Proceedings of The 37th International Conference on Algorithmic Learning Theory},
  pages = 	 {1--22},
  year = 	 {2026},
  editor = 	 {Telgarsky, Matus and Ullman, Jonathan},
  volume = 	 {313},
  series = 	 {Proceedings of Machine Learning Research},
  month = 	 {23--26 Feb},
  publisher =    {PMLR},
  url = 	 {https://proceedings.mlr.press/v313/wu26a.html}
}

\end{document}